\documentclass[12pt]{article}
\usepackage[utf8]{inputenc}
\usepackage{fullpage}

\usepackage{amsmath, amsxtra, amsfonts, amssymb, amstext, mathtools}
\usepackage{amsthm}
\usepackage{nicefrac}
\newtheorem{theorem}{Theorem}
\newtheorem{lemma}{Lemma}
\usepackage{xspace}
\usepackage[noadjust]{cite}
\usepackage{url}
\usepackage{graphics}
\usepackage[usenames,dvipsnames]{xcolor}
\usepackage[colorlinks]{hyperref}
\definecolor{linkblue}{rgb}{0.1,0.1,0.8}
\hypersetup{colorlinks=true,linkcolor=linkblue,filecolor=linkblue,urlcolor=linkblue,citecolor=linkblue}
\usepackage[algo2e,ruled,vlined,linesnumbered]{algorithm2e}

\newcommand\myatop[2]{\genfrac{}{}{0pt}{}{#1}{#2}}

\newcommand{\R}{\mathbb{R}}

\renewcommand{\epsilon}{\varepsilon}
\newcommand{\eps}{\epsilon}

\DeclareMathOperator{\E}{E}

\DeclareMathOperator{\new}{new}

\newcommand{\Bold}{B_{\old}}
\DeclareMathOperator{\old}{old}
\DeclareMathOperator{\opt}{opt}
\newcommand{\xold}{x^{\old}}
\newcommand{\xopt}{x^{\opt}}
\newcommand{\fold}{f^{\old}}

\newcommand{\xoff}{x^{\gamma+1}}

\newcommand{\assign}{\leftarrow}

\newcommand{\oea}{$(1 + 1)$~EA\xspace}
\newcommand{\algo}{$(\gamma+1)$~REA\xspace}

\newcommand{\onemax}{\textsc{OneMax}\xspace}

\newcommand{\leadingones}{\textsc{LeadingOnes}\xspace}

\newcommand{\binval}{\textsc{BinaryValue}\xspace}

\begin{document}

\title{Fast Re-Optimization via Structural Diversity}

 \author{Benjamin Doerr$^{1}$,  
 Carola Doerr$^{2}$, 
 Frank Neumann$^{3}$}
 
 \date{
 $^1$\'Ecole Polytechnique, CNRS, Laboratoire d'Informatique, Palaiseau, France\\
 $^2$ Sorbonne Universit\'e, CNRS, Laboratoire d'Informatique de Paris 6, LIP6, 75005 Paris, France\\
 $^3$ The University of Adelaide, Australia
 \\[1cm]
 \today
 }

\maketitle

\begin{abstract}
  When a problem instance is perturbed by a small modification, one would hope to find a good solution for the new instance by building on a known good solution for the previous one. Via a rigorous mathematical analysis, we show that evolutionary algorithms, despite usually being robust problem solvers, can have unexpected difficulties to solve such re-optimization problems. When started with a random Hamming neighbor of the optimum, the (1+1) evolutionary algorithm takes $\Omega(n^2)$ time to optimize the LeadingOnes benchmark function, which is the same asymptotic optimization time when started in a randomly chosen solution. 
	There is hence no significant advantage from re-optimizing a structurally good solution. 
  
  We then propose a way to overcome such difficulties. As our mathematical analysis reveals, the reason for this undesired behavior is that during the optimization structurally good solutions can easily be replaced by structurally worse solutions of equal or better fitness. We propose a simple diversity mechanism that prevents this behavior, thereby reducing the re-optimization time for LeadingOnes to $O(\gamma\delta n)$, where $\gamma$ is the population size used by the diversity mechanism and $\delta \le \gamma$ the Hamming distance of the new optimum from the previous solution. We show similarly fast re-optimization times for the optimization of linear functions with changing constraints and for the minimum spanning tree problem. 
\end{abstract}

\sloppy{
\section{Introduction}

Evolutionary algorithms have been applied to many real-world problems in important areas such engineering~\cite{DBLP:books/daglib/0034477} and supply chain management~\cite{DBLP:books/sp/chiong12}. The underlying optimization problems arising in these real-world applications are usually not static, but have dynamic and stochastic components. Due to the ability to adapt to changing environments, evolutionary algorithms have been applied to various stochastic and dynamic problems~\cite{DCOPS,DBLP:journals/swevo/RakshitKD17}.
Furthermore, many optimization problems faced in practice occur repeatedly, with slight variations in the precise instance data. Instead of solving these instances from scratch, it is common practice to start the optimization in a solution that showed good quality for previously solved problems~\cite{DBLP:journals/algorithmica/SchieberSTT18,DBLP:conf/birthday/Zych-Pawlewicz18}.

Theoretical investigations regarding the behavior of evolutionary algorithms and other bio-inspired algorithms have been carried out for different types of dynamic problems (see \cite{DBLP:journals/corr/abs-1806-08547} for an overview). This includes the MAZE problem for which difference in terms of performs for simple evolutionary algorithms and ant colony optimization approaches have been pointed out. In the context of dynamic constraints, linear functions with dynamically changing linear constraints have been investigated~\cite{shi2017reoptimization}. These investigations have been extended experimentally to the knapsack problem with a dynamic constraint bound. In addition, a general study of a simple evolutionary multi-objective approach for general costs functions with dynamic constraints has been provided in \cite{DBLP:journals/corr/abs-1811-07806}, which analyses the approximation behaviour of the algorithm in terms of the submodularity ratio of the problem.
Other important studies included investigations on dynamic makespan scheduling~\cite{DBLP:conf/ijcai/NeumannW15}, dynamic shortest paths~\cite{DBLP:journals/tcs/LissovoiW15} and variants of the dynamic vertex cover problem~\cite{DBLP:conf/gecco/PourhassanGN15,DBLP:conf/gecco/ShiNW18}. A general method to analyze the runtime of evolutionary algorithms in dynamic contexts has been given in~\cite{Dang-NhuDDIN18}.

With this paper, we contribute to the theoretical understanding of evolutionary algorithms when dealing with re-optimization problems. 
As dynamic problems change over time, a previously high quality solution $\xold$ may become unsuitable after a dynamic change has happened. We assume that a user of the algorithm is aware of the fact that a change has occurred. This is in contrast to classical dynamic problems where often the algorithm has to deal with changes automatically during the run and has to adapt to the changed problems.
However, it should be noted that evolutionary algorithms for dynamic problems often incorporate a change detection mechanism~\cite{DBLP:series/isrl/RichterY13}.
After a change, the solution $\xold$ might still be structurally quite close to a solution that is of high quality after the dynamic change has occurred. This is especially the case if only a few components have changed. Previously examined approaches have indirectly build on this by using a multi-objective formulation of the given problem where the constraint is treated as an additional objective~\cite{shi2017reoptimization,DBLP:conf/ppsn/Roostapour0N18,DBLP:journals/corr/abs-1811-07806}.

We explore the use of a previously good solution in a more direct way by proposing a population-based approach that directly searches for improvements close to the previously best solution $\xold$.
In our studies, we consider problems where the dynamic change is quantified by a parameter $\delta$.  It is often desirable not to deviate from a previously chosen solution that much in terms of design parameters as such changes might be difficult to implement. Therefore, we search for solutions after a given change has occurred that are close to the solution in the decision space prior to the change. We present a simple evolutionary algorithm called \algo. It works with a diverse set of solutions at Hamming distance at most $\gamma$ from a previously good solution $\xold$, where $\gamma$ is a parameter of the algorithm. In order to have global search capabilities, it also keeps the best solution found for the considered problem at a time. The population of \algo contains for each $i$, $0 \leq i \leq \gamma$, the best-so-far solution at Hamming distance $i$ to $\xold$. With this diversity mechanism, we aim at putting a stronger emphasis on exploring the neighborhood of the previous best solution.

We show the effectiveness of our approach on a wide range of optimization problems by rigorous runtime analyses~\cite{auger2011theory,BookNeuWit,jansen2013perspective}. Our analyses use common rigorous techniques from this area of research to show the working principles of our proposed method.

We start by investigating the classical LeadingOnes problem and consider the scenario that the problem is perturbed by flipping $\delta$ bits of the target bit string. We show that a solution of fitness at least as high as the best possible solution within Hamming distance $i \leq \gamma+1$ to $\xold$ is computed in expected time less than or equal to $2e(\gamma+1)in$. For Hamming distances $i>\gamma+1$, we bound the expected time to find such a solution from above by $2en^2$. Furthermore, we show a lower bound of $\Omega(n^2)$ for computing an optimal solution at Hamming distance $\delta \in [\gamma+2, n]$, that is, when the optimal solution is (just a little) further away from the starting solution than Hamming distance $\gamma$. This lower bound also holds when re-optimizing with the classic \oea. These lower bounds show that it is indeed the proposed diversity mechanism that makes the difference between an easy re-optimization and a re-optimization that is not faster than optimizing from a random solution. 

We then investigate the effectiveness of our approach on a constraint optimization problem where the constraint bound changes. Investigating our algorithm on the class of linear functions with a uniform constraint, we show that it re-computes an optimal solution in expected time $O(\gamma \delta n)$, where $\delta$ is the amount by which the constraint bound changes.

Finally, we investigate the minimum spanning tree problem. This classical combinatorial optimization problem has been subject to a wide range of theoretical investigations in the area of runtime analysis of bio-inspired computing~\cite{Witt14mst,ReichelS09,DBLP:journals/tcs/NeumannW07,NeumannW06,RaidlKJ06tec}. We consider a dynamic version of the problem where either $\delta$ edges are added or removed from the current graph. Our results show that \algo is able to recover an optimal solution in time $O(\gamma \delta   n)$ in both situations.

The paper is structured as follows. We introduce the algorithm and setting for dynamic changes in Section~\ref{sec2}. In Section~\ref{sec3}, we present our results for re-optimizing the \leadingones problem. We analyze linear functions with a dynamic uniform constraint in Section~\ref{sec4}. We present the results for re-optimizing the minimum spanning tree problem in Section~\ref{sec5} and finish with some concluding remarks.

\section{The \texorpdfstring{$(\gamma+1)$}{(c+1)} Re-Optimization EA}
\label{sec2}

Our algorithm, the $(\gamma+1)$~Re-Optimization EA (REA), has as input a user-specified solution $\xold$. We typically assume that $\xold$ was a solution of high quality for the function $\fold$. 

We are concerned in this work with the situation in which the function $\fold$ is perturbed by some change, resulting in a new objective function $f$. We study the time needed to recover a solution of quality at least $\fold(\xold)$. That is, in the context of maximization problems, we study the number of function evaluations that are needed by \algo to generate a solution $y$ with $f(y)\ge \fold(\xold)$, and in the context of minimization problems we require a solution $y$ with $f(y)\le \fold(\xold)$. 

Note that in this work we study both maximization and minimization problems. Algorithm~\ref{alg:algo} summarizes the \algo for maximization problems; we will describe it below. For minimization problems, only three changes are necessary: the $f^i$ are initialized by $\infty$ in line~\ref{line:infty}, and the $\ge$-signs in lines~\ref{line:selection} and~\ref{line:update} need to be exchanged for a $\le$-sign.

We quantify the difference between the old function $\fold$ and the new function $f$ by a parameter $\delta$, which denotes the smallest distance at which a solution of quality at least $\fold(\xold)$ exists. That is, there exists a solution $y$ at Hamming distance $H(y,\xold)=\delta$ for which $f(y)\ge \fold(\xold)$ and for all solutions $y'$ with $H(y',\xold)<\delta$ it holds that $f(y')<\fold(\xold)$. In our applications we assume that an upper bound $\gamma \ge \delta$ of this perturbation is known to the user (and set $\gamma=n$ otherwise).  

\begin{algorithm2e}[t]%
\textbf{Input:} Solution $\xold$\;
\textbf{Initialization:}\\
\Indp
	$x^{0},x^{*} \assign \xold$\;
	\label{line:infty}\lFor{$i=1,2,\ldots,\gamma+1$}{$x^i \assign \text{undefined}$, $f^{i} \assign -\infty$}
 \Indm   
\textbf{Optimization:}
\For{$t=1,2,3,\ldots$}{
	Select parent $x$ 
by choosing $x^{*}$ with probability $1/2$ and uniformly at random from $\{x^i \mid i \in [0..\gamma+1]\} \setminus \{x^{*}\}$ otherwise\;
	Create $y$ from $x$ by flipping in each bit independently with probability $1/n$; // standard bit mutation\\
	\lIf{$f(y) \ge f(x^*)$}{\label{line:selection}$x^{*} \assign y$}
   $i \assign \min\{H(y,\xold),\gamma+1\}$\;
	\lIf{$f(y)\ge f^i$}{\label{line:update}$x^i \assign y$, $f^i \assign f(y)$}
}
\caption{The \algo for the re-optimization (here: maximization) of a function $f:\{0,1\}^n \to \R$, which emerged from the function $\fold$ by a dynamic change.}
\label{alg:algo}
\end{algorithm2e}

Our algorithm stores for each $i$, $i \in [\gamma]:=\{1,2,\ldots,\gamma\}$, one best-so-far solution $x^{i}$ of Hamming distance $i$ to $\xold$. For notational convenience we define $x^0:=\xold$. In order to advance the search beyond the radius of $\gamma$ (e.g., if we risk that the upper bound $\gamma$ is too small), the algorithm also stores an additional search point $\xoff$ which is the best-so-far solution of Hamming distance greater than $\gamma$ to $\xold$. The points $x^i$, $i \in [\gamma+1]$ are initialized as undefined, the best function value at distance $i$, $f^i$, as $-\infty$. 

In every iteration the \algo first selects a parent individual $x$ from which one offspring $y$ will be generated. The parent is chosen through a biased random selection. With probability $1/2$ we select as $x$ the search point $x^{*}$ with the best-so-far objective value. We choose $x$ uniformly at random from $\{x^i \mid i \in [0..\gamma+1]:=\{0\} \cup [\gamma]\}\setminus \{x^*\}$ otherwise. That is, each $x^i \neq x^*$ is selected with probability $1/(2(\gamma + 1))$. A new solution candidate $y$ is created from the selected parent $x$ by standard bit mutation with mutation rate $p=1/n$. If the Hamming distance $i=H(y,\xold)$ of $y$ to $\xold$ is at most $\gamma$ this offspring $y$ replaces the previous best individual $x^{i}$ at distance $i$ if it is at least as good, i.e., if and only if $f(y) \ge f(x^{i})$. For offspring $y$ with $H(y,\xold)>\gamma$, the selection is made between $y$ and $x^{\gamma+1}$, by applying the same selection rules as in the case $i \le \gamma$.

%
%

Note that, despite the name, the \algo maintains a population size of size $\gamma+2$. We use \algo for notational convenience.

The biased parent selection of the \algo is meant to avoid too severe slow-downs when the upper bound $\gamma$ of the perturbation is large. With uniform parent selection the slowdown caused by sub-optimal search points can be as large as proportional to the population size $\gamma+2$. With the biased selection, in contrast, each step simulates, with probability $1/2$, a regular \oea. 

The advantages of storing the points $x^i$, $i \in [0..\gamma+1]$ will be motivated in the next section, using the example of re-optimizing the \leadingones problem as illustration.

\section{Re-Optimizing \texorpdfstring{\leadingones}{LeadingOnes}}
\label{sec3}

%

As a first example to demonstrate the working principles of the \algo we regard the \leadingones problem, one of the most classical benchmark problems in the theory of evolutionary computation. It has the characteristic property that the decision variables can only be optimized sequentially, that is, only when the first $i$ variables are set to the optimal value the $(i+1)$-st variable has an influence on the fitness. Such behaviors are common in non-artificial problems, see, e.g., the examples in~\cite[Section~4]{DoerrHK11} or~\cite[Section~3.2]{DoerrHK12} for two different shortest path problems.

For a ``target string'' $z \in \{0,1\}^n$ and a permutation $\sigma$ of the index set $[n]$, the \leadingones function $f_{z,\sigma}$ is defined via 
\[f_{z,\sigma}(x):=\max \{ j \in [0..n] \mid \forall k \in [j]: x_{\sigma(k)}=z_{\sigma(k)}\}\] 
for all $x \in \{0,1\}^n$. Our aim is maximizing the \leadingones functions. We note that $z$ is the unique global maximum of $f_{z,\sigma}$, regardless of $\sigma$. This problem is referred to as \leadingones because traditionally only the non-permuted instance $f_{(1,\ldots,1),\text{id}}$ with target string $(1,\ldots,1)$ was studied. This function simply returns the number of initial (\emph{leading}) ones of each solution candidate. Many EAs, and including our \algo, show exactly the same performance on any of the instances $f_{z,\sigma}$ and it thus suffices to study this particular instance with $z=(1,\ldots,1)$ and $\sigma$ being the identity function~$\text{id}$. 

When perturbing the \leadingones function, small changes can result in large changes in fitness: if we assume that $\xold$ is an optimal solution for $f_{z,\sigma}$, i.e., $\xold=z$, then changing $z$ to $z'$ by flipping the $i$-th bit of $z$ gives a new fitness value $f_{z',\sigma}(\xold)=i-1$. We also note that the new optimal solution, which is $z'$, is at Hamming distance one of $\xold$. However, all the solutions which differ from $\xold$ only in positions of index greater than $i$ have the same fitness value $i-1$. In consequence, the \oea performs a random walk on this plateau until it eventually flips the $i$-th bit. When $i$ is small, it is likely that at this point the \oea has lost track of the previously good entries in the positions $j>i$, so that it then has to recover significant parts of the tail of $z$. This unfavorable behavior of the \oea motivates our decision to store for each Hamming distance $i \in [0..\gamma]$ a best-so-far solution $x^i$, and to assign positive probability of selecting $x^i$ as parent individual even if $f(x^i)$ is strictly smaller than the current-best fitness $f(x^*)$. In the situation described above, in which only the $i$-th bit has been flipped, the \algo always has a chance of at least $1/(2(\gamma+1))$ of selecting $\xold$ as parent individual. Conditioning on $\xold$ being the selected parent, the probability to sample as offspring the new optimal solution $z'$ is  at least $1/(en)$, since exactly the $i$-th bit needs to be flipped. The expected optimization time of the \algo is hence at most $2e(\gamma+1)n$, whereas the expected re-optimization time of algorithms not using the diversity mechanism can be considerably larger, cf. Lemma~\ref{lem:oea}. 

\textbf{Summary of our results for \leadingones.} In the remainder of this section we formalize the observations made above. In Section~\ref{sec:LO-upper} we prove an upper bound for the expected re-optimization time of the \algo on \leadingones, which in particular shows that the \algo finds the best solution that is in distance $i \le \gamma+1$ from $\xold$ in time $O(\gamma i n)$ only. We complement this results by a lower bound for the performance of the \oea (Lemma~\ref{lem:oea} in Section~\ref{sec:LO-lower}), which shows that for this algorithm the re-optimization times are $\Omega(n^2)$ when the fitness of $\xold$ is at most $n/2$, even when the Hamming distance of $\xold$ and $\xopt$ is small. In particular, when started with a random Hamming neighbor of the optimum, the \oea still needs $\Omega(n^2)$ iterations to find the optimum. 

These bounds show that the \algo is significantly faster in solving re-optimization problems of the \leadingones type. We also provide a lower bound for the \algo (Theorem~\ref{thm:LOlower}) which shows that $H(\xold,\xopt) \le \gamma+1$ is a necessary condition for a fast re-optimization: Already from $H(\xold,\xopt) = \gamma+2$ on the \algo can have an at least quadratic expected re-optimization time. Upper and lower bounds thus illustrate the trade-off between choosing a too large $\gamma$, which results in a slow-down that is linear in $\gamma$, and a too small $\gamma$, which results in an at least quadratic re-optimization time. 
 On the other hand, our results also prove that the \algo with a too small $\gamma$ still has an expected optimization time of at most $2en^2$, which is not much worse than the known $\tfrac 12 e n^2$ upper bound for the \oea. 

\subsection{Upper Bound for LeadingOnes}
\label{sec:LO-upper} 

Theorem~\ref{thm:LOupper} provides an upper bound for the re-optimization time of the \algo on the \leadingones problem in which the target string has been modified from $\xold$ to $\xopt$. Both situations of an accurate and a too small upper bound $\gamma$ on the perturbation $\delta = H(\xold,\xopt)$ are covered by this bound. More precisely, the theorem shows that regardless of $\gamma$ and $\delta$ the \algo has an expected re-optimization time that is at most quadratic. When $\gamma \ge \delta-1$ the expected re-optimization time is $O(\gamma n)$. Since for this problem it provides no additional difficulties, we not only compute the expected runtimes, but we follow the approach suggested in~\cite{Doerr18evocop} and first show a domination statement and then derive from that the expected runtime and a tail bound.

\begin{theorem}\label{thm:LOupper}
  Let $f$ be a generalized \leadingones function with unique optimum $\xopt$. Let $\xold \in \{0,1\}^n$. For all $i \in [0..H(\xold,\xopt)]$, let $T_i$ be the the number of function evaluations that the \algo needs to find a solution $y$ with $f(y) \ge \max\{f(y') \mid y' \in \{0,1\}^n, H(y',\xold) \le i\}$. 
  \begin{enumerate}
  \item If $i \le \gamma+1$, then $T_i$ is dominated by a sum of $i$ independent geometric distributions with success rate $\frac{1}{2e (\gamma+1) n}$. Consequently, 
  \begin{align*}
  &E[T_i] \le 2e(\gamma+1)i n =: \mu^+,\\ 
  &\Pr[T_i \ge (1+\eps) \mu^+] \le \exp\left(-\frac{\eps^2 i}{2(1+\eps)}\right) \mbox{ for all $\eps \ge 0$}.
  \end{align*}
  \item Regardless of $i$, the time $T_i$ is dominated by a sum of $n$ independent geometric random variables with success rate $\frac 1 {2en}$. Consequently,
  \begin{align*}
  &E[T_i] \le 2en^2 =: \mu^+,\\
  &\Pr[T_i \ge (1+\eps) \mu^+] \le \exp\left(-\frac{\eps^2 n}{2(1+\eps)}\right) \mbox{ for all $\eps \ge 0$}.
  \end{align*}
  \end{enumerate}
When $\gamma \ge \delta-1$, the expected re-optimization time of the \algo on the modified \leadingones function is thus at most $\min\{2e(\gamma+1)\delta n, 2en^2\}$, provided that $\xold$ was an optimal solution for $\fold$. 
\end{theorem}

\begin{proof} 
  By the symmetry of all operators used in the \algo, we can assume without loss of generality that the optimum of $f$ is $\xopt  = (1,\dots,1)$.
  Let $0 \le i \le H(\xold,\xopt)$. We first consider the case that $i \le \gamma+1$. Due to the nature of the \leadingones function, there is a unique search point $x^{i,*}$ in $\{y \in \{0,1\}^n \mid H(y,\xold) = i\}$ with maximal fitness. This search point is equal to $\xold$ in all bit positions except the first $i$ positions in which $\xold$ is zero. Hence $H(x^{i,*},\xold) = i$. If $i \le \gamma$, then let  $T^*_i$ be the iteration in which the program variable $x^i$ takes the value $x^{i,*}$. For $i = \gamma+1$, let $T^*_i = T_i$. Note that $T_i^*$ stochastically dominates $T_i$ for all $i \le \gamma+1$, so it suffices to show our claim for $T_i^*$ instead of $T_i$. 
  
  We use a fitness level argument to estimate $T^*_i$. If $i \le \gamma$, then for all $j \le i$, we say that the algorithm is in state $j$ if $x^j = x^{j,*}$ and, if $j < i$, also $x^{j+1} \neq x^{j+1,*}$ holds. For $i = \gamma+1$, we say that the algorithm is in state 
	\begin{itemize}
		\item $j < \gamma$ if $x^j = x^{j,*}$ and $x^{j+1} \neq x^{j+1,*}$, 
		\item $j=\gamma$, when $x^j = x^{j,*}$ and $f(x^*) < \max\{f(y') \mid y' \in \{0,1\}^n, H(y',\xold) \le \gamma+1\}$, 
		\item $j=\gamma+1$ if $f(x^*) = \max\{f(y') \mid y' \in \{0,1\}^n, H(y',\xold) \le \gamma+1\}$.
	\end{itemize}
  
  In both cases $i \le \gamma$ and $i = \gamma+1$, we see that when the algorithm is in state $j < i$, then with probability at least $\frac{1}{2(\gamma+1)}$ it chooses $x^j$ as parent and (in this case) with probability $\frac 1n (1-\frac 1n)^{n-1} \ge \frac 1 {en}$, 
  flips exactly the unique bit that $x^{j,*}$ and $x^{j+1,*}$ differ in. Hence each iteration in state $j$ has a probability of at least $\frac{1}{2(\gamma+1)en}$ of ending in a higher state. By the classic fitness level theorem~\cite{Wegener01} we obtain the bound $E[T^*_i] \le 2e(\gamma+1)i n$. By the domination version of the fitness level theorem~\cite[Theorem~2]{Doerr18evocop} we also obtain the domination result and the tail bound, for the latter using a Chernoff bound for sums of independent geometric random variables~\cite[Theorem~3(i)]{Doerr18evocop}.

When $i$ is arbitrary, and not necessarily at most $\gamma+1$, we can still apply the classic fitness level method to the fitness of the best solution $x^*$. Note that when $x^*$ has some fitness $j$, then with probability at least $\frac 12$ this $x^*$ is chosen as parent and with probability $\frac 1n (1-\frac 1n)^{n-1} \ge \frac 1 {en}$ exactly the $(j+1)$-st bit is flipped. The resulting offspring $y$ has a fitness greater than $x^*$ and thus replaces $x^*$ (and possibly some $x^k$). Hence in each iteration, with probability at least $\frac 1 {2en}$ the fitness of $x^*$ increases. Again, the fitness level theorems give the claimed bounds for the time $T$ to find the optimum. Since $T$ dominates all $T_i$, the claims for $T_i$ are proven.
\end{proof}

\subsection{Lower Bound for LeadingOnes}
\label{sec:LO-lower}

We now show that the \oea without diversity mechanism can have a quadratic runtime to optimize \leadingones even when started with a Hamming neighbor of the optimum. In fact, many Hamming neighbors lead to this runtime, so this result also holds when starting with a random Hamming neighbor. Using similar arguments, we also show that the requirement $i \le \gamma+1$ in the first part of Theorem~\ref{thm:LOupper} cannot be relaxed. Already for a Hamming distance of $\gamma+2$, we have an expected quadratic optimization time. 

The reason for these high runtimes is as follows. For the \oea, the initial structurally good solution is easily replaced by other solutions of same fitness, which are structurally further away from the optimum. By this the advantage of starting with a good solution is lost. When the optimum has Hamming distance at least $\gamma+2$ from $\xold$, then (i) the \algo finds it hard to generate the optimum from one of the $x^k$, $k \in [0..\gamma]$, as these have a Hamming distance at least $2$ from the optimum, and (ii) the \algo also finds it hard to find the optimum via optimizing $x^*$ as this search point again quickly becomes structurally distant from the optimum. 

The rough reason for structurally good solutions moving away from the optimum is that bits with higher index than the current fitness plus one are neutral, that is, they are subject to mutation, but have no influence on the fitness and can therefore not bias the selection. For such bits, independent of their initial values in $\xold$, the probabilities of being zero or one converge to $1/2$. This was first shown in~\cite[proof of Theorem~10]{DoerrGHN07} and later exploited in several analyses how evolutionary algorithms optimize the \leadingones function~\cite{LassigS13,DoerrSW13foga,RoweS14,Sudholt18,BottcherDN10}. 

For the sake of completeness, we quickly repeat the statement in~\cite{DoerrGHN07} and its proof (where we note that in~\cite{DoerrGHN07} apparently the binomial coefficients were forgotten in the proof). We note that an essentially identical result (their assumption $t \ge n\ln(n)$ can be freely omitted) was independently proven in~\cite{LassigS13} with identical arguments).

\begin{lemma}\label{lem:neutral}
  Let $X_0, X_1, \dots$ be a sequence of binary random variables such that $\Pr[X_t = X_{t-1}] = 1 - \frac 1n$ and $\Pr[X_t = 1 - X_{t-1}] = \frac 1n$ independently for all $t \ge 1$. Then 
  \begin{align*}
  \Pr[X_t = X_0] = \tfrac 12 + \tfrac 12 (1 - \tfrac 2n)^t,\\
  \Pr[X_t \neq X_0] = \tfrac 12 - \tfrac 12 (1 - \tfrac 2n)^t.
  \end{align*}
\end{lemma} 

\begin{proof}
  We have 
  \begin{align*}
   \Pr[X_t = X_0] &- \Pr[X_t \neq X_0] \\
	= & \sum_{\myatop{i=0}{2|i}}^t \binom{t}{i}\left(\frac 1n\right)^i \left(1-\frac 1n\right)^{t-i} - \sum_{\myatop{i=0}{2 \nmid i}}^t \binom{t}{i}\left(\frac 1n\right)^i \left(1-\frac 1n\right)^{t-i} \\
  = & \sum_{{i=0}}^t \binom{t}{i}\left(-\frac 1n\right)^i \left(1-\frac 1n\right)^{t-i} 
	= \left(-\frac 1n + \left(1 - \frac 1n\right)\right)^t \\
	= & \left(1-\frac 2n\right)^t.
  \end{align*}
  The claims follow from $\Pr[X_t = X_0] + \Pr[X_t \neq X_0]=1$ and elementary transformations.  
\end{proof}

We are now ready to state and prove our lower bound result. For the ease of presentation, we only cover the case that $\gamma \le \frac n4-2$. Since constant factors play no important role in the proof, it is immediately clear from the proof that this condition could be relaxed to $\gamma \le (1-\eps)n$ for any constant $\eps>0$. Further, we are optimistic that the result holds for all values of $\gamma$. However, we feel that the case of values of $\gamma$ that are linear in $n$ is not interesting enough to justify the extra effort. Note that for $\gamma = \Omega(n)$ and $\delta \ge \gamma+2$ our starting solution has a linear Hamming distance from the optimum. This can hardly be seen as re-optimization from a solution close to the optimum.

\begin{theorem}\label{thm:LOlower}
  Let $\gamma \le \frac 14 n - 2$. Let $f$ be the \leadingones function with unique optimum $\xopt = (1,\dots,1)$. For all $\delta \in [\gamma+2..n]$ there is an $\xold \in \{0,1\}^n$ with $H(\xold,\xopt) = \delta$ such that the expected time the \algo started with $\xold$ takes to find the optimum of $f$ is $\Omega(n^2)$.
\end{theorem}

\begin{proof}
  Let $\xold$ be the search point defined by $\xold_1 = \dots = \xold_\delta = 0$ and $\xold_{\delta+1} = \dots = \xold_n = 1$. Note that $H(\xold,\xopt) = \delta$.
  
  Consider the first iteration $t_0$ in which the search point stored in $x^*$ reaches a fitness of at least $\gamma+2$. Let $x$ be the parent chosen (which by assumption has a fitness of at most $\gamma+1$) and let $y$ be the offspring generated in this iteration (which will end up in $x^*$). Since the algorithm as mutation operation flips bits independently with probability $\frac 1n$ and since we know $f(y) \ge \gamma+2$, we have $y_1 = \dots = y_{\gamma+2} = 1$ and all further bits are obtained from the corresponding bit of $x$ by flipping it with probability $\frac 1n$. Let $I_0 := \{i \in [\gamma+3..\lceil \frac 12 n \rceil] \mid x_i = 0\}$ and $I_1 := \{i \in [\gamma+3..\lceil \frac 12 n \rceil] \mid x_i = 1\}$. We compute
  \begin{align*}
  \Pr[f(y) \ge \tfrac 12 n] 
  & = \prod_{i=\gamma+3}^{\lceil \frac 12 n \rceil} \Pr[y_i = 1] \\
  & = \prod_{i \in I_0} \Pr[y_i = 1]  \prod_{i \in I_1} \Pr[y_i = 1] \\
  & = (\tfrac 1n)^{|I_0|}  (1-\tfrac 1n)^{|I_1|} \le (1-\tfrac 1n)^{n/2 - \gamma - 2}\\
	&\le (1-\tfrac 1n)^{n/4} 
	 \le e^{-1/4}.
  \end{align*}
  
  Let us condition on $f(y) \le \frac 12 n$ in the following (and recall that we have this event with probability at least $1 - e^{-1/4} \ge 0.2$). We first argue that we can assume that whenever in the following $0.1 n^2$ iterations we choose an $x^k$, $k \in [0..\gamma]$, as parent, then the offspring does not replace the current value of $x^*$. Since the search point stored by the algorithm in $x^k$, $k \in [0..\gamma]$, at all times has a Hamming distance of at least $\gamma+2-k \ge 2$ from any search point with fitness $\gamma+2$ or larger, the probability that an $x^k$ is mutated to a search point with fitness at least the one of $x^*$, is at most $n^{-2}$. By a simple union bound over $0.1 n^2$ iterations, we obtain that with probability at least $0.9$, in no iteration of the time interval $I = [t_0+1..t_0+0.1n^2]$ an offspring of an $x^k$ makes it into $x^*$.
  
  Taking also this assumption, we can ignore all iterations in the time interval $I$ that use a parent different from $x^*$ as they cannot generate the optimum and cannot interfere with the process on $x^*$. In the remaining iterations in $I$, the \algo simulates a \oea using $x^*$ as population. By Lemma~\ref{lem:oea} below, the first $\tfrac 1 {16} n^2$ of these iterations (or fewer, if there are fewer such iterations in $I$) with constant probability do not create the optimum. Taking this and all assumptions taken on the way together, we see that with constant probability, the \algo within $\frac 1 {16} n^2$ iterations does not find the optimum. Consequently, the expected optimization time is $\Omega(n^2)$.

\end{proof}

We finish the proof of the main result by providing the missing ingredient that the \oea with constant probability needs a quadratic number of function evaluations to optimize \leadingones even when initialized with an arbitrary search point of fitness at most $n/2$, that is, even when the initial search point is a Hamming neighbor of the optimum. This result might be of independent interest. Again, we did not try to optimize the constants, in particular, a quadratic runtime could also be shown when the initial fitness is as large as $(1-\eps) n$ for an arbitrarily small positive constant $\eps$.

\begin{lemma}\label{lem:oea}
  Consider a run of the \oea on the \leadingones function $f$, initialized with an arbitrary search point $x^0$ such $f(x^0) \le n/2$. Let $T$ be the first iteration in which an optimal solution is generated. Then $\Pr[T \le n^2/16] \le \tfrac 1e + \exp(-\Omega(n))$.
\end{lemma}

\begin{proof}
  Let $t_0$ be any number 
   such that among the first $t_0$ iterations of the \oea, in exactly $n-1$ iterations, called \emph{relevant iterations} in the following, an offspring $y$ is created that agrees with the parent $x$ in the first $f(x)$ bits. Note that the remaining $t_0 - (n-1)$ iterations create offspring worse than the parent, so that they have no influence on the optimization process except wasting time. In each relevant iteration, with probability exactly $1/n$ an offspring strictly better than the parent is generated (namely when the first missing bit is flipped). Hence with probability $(1-\frac 1n)^{n-1} \ge 1/e$, none of the $n-1$ relevant iterations creates a strict improvement. Let us condition on this event in the following.
  
  We now regard in detail the search point $x$ resulting from the first $t_0$ iterations. Put differently, we analyze the parent individual of iteration $t_0+1$. By our assumption, we have $f(x) = f(x^0)$. As discussed before, we can ignore the non-relevant iterations as they do not change the current individual. In each relevant iteration, the parent is replaced by the offspring, which was generated by flipping each bit greater than $f(x)+1$ independently with probability $1/n$. Consequently, for $i \in [f(x)+2..n]$, the value $x_i$ is obtained from $x^0_i$ by $n-1$ times independently flipping the bit value with probability $1/n$ (independently from the other bits). 
  
  This observation remains true (in an analogous fashion) for future generations $t \ge t_0$: Let $x$ be the search point at the end of some iteration $t \ge t_0$ and assume by induction that it is such that for all $i \ge f(x)+2$, the bit value $x_i$ of the $i$-th bit is obtained from $x^0_i$ by flipping it some number $n_t \ge n-1$ times independently with probability $1/n$. Let $y$ be the offspring generated (from $x$) in iteration $t+1$ and let $x'$ be the outcome of the selection between $x$ and $y$. Clearly, for all $i \ge f(x)+2$, the bit value $y_i$ of the $i$-th bit is obtained from $x^0_i$ by flipping it $n_t+1$ times independently with probability $\frac 1n$. If $f(y) < f(x)$, then $y$ is discarded and our claim holds for $x'$ since it holds for $x$. If $f(y) \ge f(x)$, then the first $f(y)+2$ bits of $y$ are determined by the fitness, but the remaining bits have no influence on the decision to continue with $y$. Hence for all $i \ge f(x)+2$, the value of $y_i$ is obtained from $x^0_i$ by flipping it $n_t+1$ times independently with probability $1/n$. Obviously, the same statement holds for $x'$ in this case. 
  
  In summary, we see that under the assumption taken initially (which holds with probability at least $1/e$), in all iterations following iteration $t_0$, the parent individual $x$ is such that all bits $i \in [f(x)+2..n]$ independently have the distribution of taking the initial bit value $x^0_i$ and flipping it some number $n_t \ge n-1$ of times independently with probability $\frac 1n$. By Lemma~\ref{lem:neutral}, independent of the initialization of the bit value and independently for all $i \ge f(x)+2$, we have $\Pr[x_i = 1] \le \frac 12 + \frac 12 (1 - \frac 2n)^{n_t} \le \frac 12 + \frac 12 (1 - \frac 2n)^{(n-1)} \le \frac 12 + \frac 12 (1 - \frac 2n)^{n/2} \le \frac 12 + \frac 1 {2e} =: p \le 0.7$. 
  
  We use this statement to describe the fitness gain in one iteration $t > t_0$. To have a positive fitness gain from a parent $x$, it is necessary that the first $f(x)$ bits do not flip and that bit $f(x)+1$ does flip (we call this event a success). This happens with probability $(1-\frac 1n)^{f(x)} \frac 1n \le \frac 1n$. Since in this case, each further bit of the offspring is one with probability at most $p$, regardless of the outcomes of the other bits, the fitness gain in case of a success is dominated by a geometric distribution with parameter $1-p$. 
  
  We finish the proof by showing that the total fitness gain $X$ in the time interval $[t_0+1..t_0+\frac 16 n^2]$ 
  with high probability is less than $\frac 12 n$ and thus not sufficient to reach the optimum (recall that the fitness after iteration $t_0$ still was the initial fitness $f(x^0) \le \frac 12 n$). Since the probability for a success is at most $\frac 1n$ regardless of what happened in the previous iterations, by Lemma~3 of~\cite{Doerr18evocop} the number of successes in these $\frac 1{16} n^2$ iterations is dominated by a sum of $\frac 1{16} n^2$ independent binary random variables with success probability~$\frac 1n$. Applying a common Chernoff bound, e.g., the simple multiplicative bound in Theorem~10.1 of~\cite{Doerr18bookchapter}, we see that with probability $1 - \exp(n/48)$, the number of successes is at most $\frac 18 n$. In this case, the fitness gain $X$ is dominated by a sum of $\frac 18 n$ independent geometric distributions with success probability $1-p$. Hence $E[X] = \frac 1 {8(1-p)} n$ and, using a Chernoff bound for geometric random variables like Theorem~3~(i) of~\cite{Doerr18evocop}, 
  \begin{align*}
	\Pr[X \ge \tfrac 12 n] & \le \Pr[X \ge 1.2 E[X]] \le \exp\left(- \frac{(0.2n)^2}{2\cdot \frac 18 n(1+0.2n/n)}\right)\\
	&= \exp(-\Omega(n)).
	\end{align*}
  Hence, under our initial assumption of having no fitness gain in the first $n-1$ iterations, with probability $1 - \exp(-\Omega(n))$ the expected runtime of the EA is more than $n^2/16$. Since the initial assumption was satisfied with probability $1/e$, the claim is proven.
\end{proof}

\section{Re-Optimizing Linear Functions with Modified Uniform Constraints}
\label{sec4}

The next example for which we analyze the performance of the \algo is a constrained optimization problem. Specifically, we study the maximization of a linear \emph{profit} function $p:\{0,1\}^n \to \R, x \mapsto \sum_{i=1}^n{w_i x_i}$ subject to the \emph{uniform constraint} $\sum_{i=1}^n x_i \le B$. The perturbation concerns the size of the uniform constraint: in the perturbed problem, the size bound $B$ is replaced by $B-\delta$ or $B+\delta$. 

As mentioned in the introduction, this problem has been previously studied in~\cite{shi2017reoptimization}, and constitutes one of the few constrained optimization problems for which the running time of EAs has been formally analyzed. Shi et al. analyzed the \emph{expected reoptimization time} of the \oea and of three multi-objective EAs. In our terminology, they thus assume that $\xold$ was an optimal solution for the problem (before the size bound $B$ had been changed), and bound the expected time needed by the EAs to identify an optimal solution $x^{*}$ for the perturbed problem. A main conclusion of the work by Shi et al. is that it can be beneficial to regard the constrained problem as a two-objective problem with the size (i.e., the number of ones) of the solution as one objective, and the profit values $p(x)$ as second objective. 

For the \oea, Shi et al. transform the constrained problem 
\begin{align}\label{eq:constrained}
\max p(x) = \sum_{i=1}^n w_i x_i \\
\nonumber \text{s.t.} \sum_{i=1}^n x_i \le B
\end{align}
into a pseudo-Boolean objective function 
\begin{align}\label{eq:linearfitness}
f:\{0,1\}^n \to \R, x \mapsto p(x) - C \max\big\{ \sum_{i=1}^n x_i-B, 0\big\},
\end{align}
where $C:=n |w_{\max}|+1$ for $w_{\max}:=\max\{ |w_i| \mid 1 \le i \le n\}$. With this choice, the penalty term guides the search towards the feasible region, which, once hit by the (1+1) EA, is not left by this algorithm, thanks to its elitist selection. It is proven in~\cite{shi2017reoptimization} that the \oea has an $O(n^2 \log(B w_{\max}))$ expected reoptimization time. 

We analyze the expected optimization time of the \algo on this problem formulation. The following theorem shows that it is $O(n \gamma\delta)$, provided that the perturbation estimate $\gamma$ satisfies $\gamma \ge \delta-1$. 
\begin{theorem}
\label{thm:linear}
Let $\fold:\{0,1\}^n \to \R, x \mapsto p(x) - C \max\{ \sum_{i=1}^n x_i-\Bold, 0\}$ be a function as in~\eqref{eq:linearfitness}, with linear profit function $p$. Let $\xold$ be an optimal solution for $\fold$ satisfying $\sum_{i=1}^n \xold_i\le \Bold$ (i.e., $\xold$ is feasible solution for the corresponding constrained problem~\eqref{eq:constrained}). 
Let $\delta$ be a positive integer satisfying $\delta \le \min\{\Bold, n-\Bold\}$, let $B \in \{\Bold-\delta, \Bold+\delta\}$, and let $f:\{0,1\}^n \to \R, x \mapsto p(x) - C \max\{ \sum_{i=1}^n x_i-B, 0\}$ be the perturbed fitness function that we obtain from $\fold$ by replacing the penalty term $C \max\{ \sum_{i=1}^n x_i-\Bold, 0\}$ by $C \max\{ \sum_{i=1}^n x_i-B, 0\}$. 

For all $i \in [\gamma+1]$ the expected number of fitness evaluations needed by the \algo to identify a solution of function value at least $\max \{ f(y) \mid H(y,\xold) \le i \}$ is $O(n \gamma i)$. 
\end{theorem} 

\begin{proof}
Let $i \le \gamma+1$. Let $T^i$ denote the random variable that counts the number of function evaluations needed by the \algo to identify an optimal solution $x^{i,*}$ at Hamming distance $i$ from $\xold$; i.e., a solution $x^{i,*} \in \arg\max \{f(y) \mid H(y,\xold)=i\}$. 

For an inductive proof, we first bound $\E[T^{1}]$. Note that at any point in time the size of the population is at most $\gamma+2$. The solution $\xold$ (which is never removed from the population) has thus a probability of at least $1/(2(\gamma+1))$ of being selected as parent individual. Conditioning on $\xold$ being selected as parent, the probability of flipping a 0-bit of maximal weight ($B>\Bold$) or a 1-bit of minimal weight ($B<\Bold$), respectively, and no other bit is at least $(1/n)(1-\frac{1}{n})^{n-1} \ge 1/(en)$. The expected waiting time for creating a point $x^1 \in \arg\max \{f(y) \mid H(y,\xold)=1\}$ is thus at most $2e n (\gamma+1) = O(n \gamma)$.

For fixed $j \in [i-1]$ assume that $x^j$ has already been updated to a point of maximal possible fitness, i.e, $f(x^{j})= \max \{f(y) \mid y \in \{0,1\}^n \text{ with } H(y,\xold)=j\}$. By the same reasoning as above, the probability to select $x^{j}$ as parent is at least $1/(2(\gamma+1))$, and the probability to flip a 0-bit (1-bit) of maximal (minimal) weight is at least $1/(en)$, showing that $\E[T^{j+1}] \leq \E[T^{j}] + 2e n (\gamma+1) = O(n \gamma (j+1))$ by the induction hypothesis. 
\end{proof}

We note, without going into great detail, that the expected reoptimization time can strongly depend on the structure of the weights of the linear profit function. For an illustrative example, let us assume that the profit function is the \binval function, i.e., the linear function with $w_i=2^{n-i}$. Assume that $\Bold=cn$ for some $c<1$. Assume further that $\xold$ is an optimal solution for $\fold$, i.e., $\xold$ is the string with entry $1$ in positions $i \le \Bold$ and entry $0$ in positions $i>\Bold$. Finally, assume that the new size bound is $B=\Bold+1$, i.e., we have $\delta=1$. Then, regardless of $\gamma$, the expected reoptimization time is at least linear in $n$, since the $B$-th bit needs to be flipped in order to obtain the unique optimal solution for the perturbed function, which is the string having the first $B$ entries equal to one, and all others equal to zero. If, on the other hand, the linear profit function is the \onemax function (i.e., the linear functions with $w_1=w_2=\ldots=w_n=1$), the expected reoptimization time for the same perturbation of the size bound is constant. More precisely, it suffices to select $\xold$ as parent and to flip in it exactly one of the $(1-c)n$ zero-bits. Unless a new optimal solution has already been found, the probability that a solution with fitness $\fold$ is selected as parent equals $1/2$, regardless of $\gamma$. The probability to create an optimal solution for the new problem instance is thus at least $(1/2)(1-c)n/(en)=\Theta(1)$. This example can easily be extended to many other situations. A more detailed discussion of these effects can be found in~\cite{shi2017reoptimization}. 

\section{Minimum Spanning Trees}
\label{sec5}

The classical minimum spanning tree (MST) problem can be formulated as follows.
Given an edge-weighted undirected graph $G=(V,E,w)$, with $n=|V|$ nodes and $m=|E|$ edges, the goal is to find a subset $E' \subseteq E$ of minimal cost such that the graph $G(V,E')$ is connected. We denote by $w_i$ the weight of edge $e_i$, $1 \leq i \leq m$, and assume that weights are strictly positive. 
We consider the search space $\{0,1\}^m$ where $x \in \{0,1\}^m$ gives a selection of edges, i.e., edge $e_i$ is selected if and only if $x_i=1$.

We consider the fitness function $f(x) = (c(x), w(x))$ where $c(x)$ denotes the number of connected components of the graph given by $x$ and $w(x)=\sum_{i=1}^m w_i x_i$ is the weight of the chosen edges. The fitness function should be minimized with respect to lexicographic order which is equivalent to assigning to each additional components a large penaly. This standard formulation of the MST problem has already been investigated in~\cite{DBLP:journals/tcs/NeumannW07} and a multi-objective formulation trading off $c(x)$ and $w(x)$ against each other has been considered in~\cite{NeumannW06}. 

\subsection{Additional Edges}
\label{sec:MSTadd}

We first consider the case where $\delta$ edges $e \not \in E $ are added to the graph $G=(V,E)$. Let $x^*=(x_1^*, \ldots, x_m^*)$ be a solution representing a minimum spanning tree before a change has occurred.

We assume that the number (but not necessarily the endpoints, nor the weights) of the additional edges is known. The size of the search space is thus increased by $\delta$. The new edges are labeled $e_{m+1}, \ldots e_{m+\delta}$. The \algo uses as input for the reoptimization the search point $\xold=(x^*_1, \ldots x^*_m,x_{m+1}, \ldots x_{m+\delta})$ with $x_{m+1}= \ldots= x_{m+\delta}=1$. 

Since the number of edges in an MST of a graph with $n$ vertices is $n-1$, we know in this setting---unlike the cases considered in the previous sections---that the actual perturbation is $\delta$. We nevertheless assume that a general bound $\gamma \ge \delta$ is used in the \algo, since $\xold$ may have been communicated without the size bound $\delta$. 

We start our performance analysis of \algo with some structural observations about the minimum spanning tree in the extended graph.
\begin{lemma}
\label{lem:opttree}
Let $T$ be a minimum spanning tree for a given graph $G=(V,E,w)$ and let $G_{\delta}$ be obtained from $G$ by adding a set $E_{\delta}$ consisting of $\delta$ edges that satisfy $E \cap E_{\delta} = \emptyset$. Then there exists a minimum spanning tree $T^*$ of $G_{\delta}$ that can be obtained by removing 
$\delta$ edges from $T \cup E_{\delta}$.
\end{lemma}

\begin{proof}
Let $T^*$ be a minimum spanning tree of $T \cup E_{\delta}$. Obviously $T^*$ can be obtained from $T \cup E_{\delta}$ by deleting $\delta$ edges. We show that $T^*$ is also a minimum spanning tree of $G_{\delta}$. As $T$ is a minimum spanning tree of $G$, for any cut that is obtained by removing an edge $e$ of $T$, there is no edge in $E \setminus T$ that has a smaller weight than $e$ and connects the two components (as T is a minimum spanning tree of $G$). Hence, only edges of $E_{\delta}$ can result in a spanning tree of smaller weight than $T$ and a minimum spanning tree of $T \cup E_{\delta}$ is also a minimum spanning tree of $G_{\delta}$.
\end{proof}

According to Lemma~\ref{lem:opttree}, we can obtain a minimum spanning tree of $G_{\delta}$ by deleting $\delta$ edges of $T \cup E_{\delta}$ in decreasing order of their weights that do not disconnect the graph.

\begin{lemma}\label{lem:hami}
Let $T$ be a minimum spanning tree for a given graph $G=(V,E,w)$ and let $G_{\delta}$ be obtained from $G$ by adding a set $E_{\delta}$ consisting of $\delta$ edges not previously present in $E$. An optimal solution $x^{i}$, $1 \leq i \leq \delta$, representing a connected graph of Hamming distance exactly $i$ to $\xold$ is obtained from $T \cup E_{\delta}$ by sequentially removing exactly $i$ edges of the largest weight such that the graph does not get disconnected. 

Furthermore, an optimal solution $x^{i+1}$ at Hamming distance $i+1$ from $x^{\old}$ can be obtained from $x^i$ by removing the largest edge whose removal does not make the graph disconnected.
\end{lemma}

\begin{proof}
Consider the graph $G_{\delta}$. According to Lemma~\ref{lem:opttree} we can obtain a minimum spanning tree $T^*$ for $G_{\delta}$ by removing those $\delta$ edges from $T \cup E_{\delta}$ that are of largest weight and whose removal does not make the graph disconnected. It suffices to remove these edges sequentially, in decreasing order of weight. 

Let $E^*=\{e_1, \ldots, e_{\delta}\} \subseteq T \cup E_{\delta}$, with $w(e_1) \geq \ldots \geq w(e_{\delta})$ be such a set of $\delta$ edges whose removal yields a minimum spanning tree. For $1 \le i \le \delta$ set $E_i:=\{e_1, \ldots, e_i\}$; i.e., $E_i$ is a subset of $E^*$ consisting of a set of edges having the $i$ largest weights (ties are broken arbitrarily). Let $x^i$ be the string obtained from $\xold$ by flipping those bits that correspond to the edges $E_i$. The Hamming distance $H(\xold,x^i)$ of $x^i$ to $\xold$ is $i$, and the weight of $x^i$ equals $w(x^i) = w(x^{\old}) - w(E_i)$. 

We show that $x^i$ is an optimal solution at Hamming distance $i$ to $\xold$, in the sense that $c(x^i)=1$ and $w(x^i)=\min\{w(y) \mid H(\xold,y)=i, c(y)=1 \}$. 
Let $y^i$ be a solution representing a connected graph at Hamming distance $i$ to $\xold$ having the smallest weight among all such solutions. 
Since $c(x^i)=1=c(y^i)$, we only need to show that $w(x^i) = w(y^i)$. If $y^i$ is obtained by removing exactly $i$ edges from $\xold$ then we have $w(x^i)=w(y^i)$ due to the construction of $x^i$. If $y^i$ is not obtained by removing exactly $i$ edges, then an additional edge $e \not \in E_{\delta}$ has to be inserted. In this case, by the restriction that $H(y^i,\xold)=i$, there can be at most $i-1$ edges that have been removed from $\xold$, and these edges have to be such that the resulting graph is not disconnected. This implies $w(y^i) \geq w(x^{\old})- w(E_{i-1})>w(x^i)$, contradicting the choice of $y^i$. We therefore obtain that $y^i$ is obtained from $\xold$ by removing $i$ edges and therefore we have $w(x^i)=w(y^i)$.  

Hence, the optimal solution $x^i$ is a solution obtained from $\xold$ by removing a set $E_i$ of largest weight that does not make the graph disconnected. Having reached an optimal solution $x^i$ which is a subset of the edges chosen by $\xold$, an optimal solution $x^{i+1}$ is obtained by flipping the $1$-bit corresponding to the largest edge whose removal does not make the graph disconnected.
\end{proof}

\begin{theorem}
Let $\gamma \geq \delta$. Then the expected time until \algo has computed a minimum spanning tree after the addition of $\delta$ edges to the graph $G=(V,E,w)$ when starting with $x^{\old}$ is $O(\gamma \delta n)$.
\end{theorem}

\begin{proof}
Let $G_{\delta}$ be the graph obtained form $G$ by the addition of the $\delta$ edges.

As in the previous two sections we perform an inductive proof and show that for each $1 \le i \le \delta$ the expected number of iterations needed by the \algo to obtain an optimal solution $x^i$ at Hamming distance $i$ from $\xold$ is $O(\gamma i n)$. Let $x^0=\xold$ and let $i \in [0..\delta-1]$ be such that the \algo has found an optimal solution at Hamming distance $i$ from $\xold$. By Lemma~\ref{lem:hami} an optimal solution $x^{i+1}$ at Hamming distance $i+1$ from $\xold$ is obtained from $x^i$ by flipping in it exactly one of those $1$-bits that correspond to an edge of largest weight and which is such that its removal does not make the graph disconnected, and flipping no other bit in $x^i$. The solution $x^i$ is selected with probability at least $1/(2(\gamma+1))$ and the corresponding mutation happens with probability at least $1/(en)$, so that the expected time needed to create from $x^i$ an optimal point $x^{i+1}$ at Hamming distance $i+1$ is $O(\gamma n)$. Since the value of $i$ has to be increased at most $\delta$ times, a minimum spanning tree $T^*$ of $G_{\delta}$ is obtained from $\xold$ in expected time $O(\gamma \delta n)$.
\end{proof}

\subsection{Removal of Edges}
\label{sec:MSTremove}

We now consider the case where a set of $\delta$ edges $E_{\delta} \subset E$ is removed from the graph $G=(V,E,w)$ such that 
a still connected graph
$G_{\new}=(V,E\setminus E_{\delta},w_{\new})$ is obtained (where $w_{\new}$ denotes the restriction of $w$ to the edges in $E\setminus E_{\delta}$). Let $x^*=(x^*_1, \ldots, x^*_m)$ be a solution representing a minimum spanning tree of $G$. We remove the bits corresponding to the removed edges in order to obtain the solution $x^{\old}$ that we are using for the initialization of the reoptimization process. Without loss of generality and to ease the presentation, we assume that the last $\delta$ bits are removed, which implies $x^{\old} = (x^*_1, \ldots, x^*_{m-\delta})$. 

\begin{theorem}
Let $\gamma \geq \delta$. Then the expected time until the \algo has computed a minimum spanning tree after the removal of $\delta$ edges from $G=(V,E,w)$ when starting with $x^{\old}$ is $O(\gamma \delta n)$.
\end{theorem}

\begin{proof}
Let $\delta' \leq \delta$ be the number of edges that have been removed from the minimum spanning tree represented by $x^*$ for $G=(V,E,w)$.
The solution $x^{\old}$ is a minimum spanning forest of $G_{\new}=(V,E\setminus E_{\delta},w_{\new})$ consisting of $\delta'+1$ 
connected components.
For our analysis, we always pick the solution $x^i$ such that for $x^i$ and all $x^j$, $0 \leq j \leq i< \delta'$, a minimum spanning forest with $\delta'-j+1$ 
connected components for $G_{\new}$ has already been obtained. 
Solution $x^i$ is chosen as a parent for mutation with probability at least $1/(2(\gamma+1))$. 
It is well known (and used, for example, in Prim's algorithm) that flipping the bit corresponding to an edge of minimal weight that does not create a cycle produces a solution $x^{i+1}$ that is a minimum spanning forest with $\delta'-i$ 
connected components at Hamming distance $i+1$ to $x^{\old}$. There are at most $\delta'$ steps in which the value of $i$ has to be increased such that a minimum spanning tree for $G_{\delta}$ which has Hamming distance $\delta' \leq \delta$ to $x^{\old}$
is obtained. This implies that a minimum spanning tree for $G_{\delta}$ is obtained after an expected number of $O(\gamma \delta' n) = O(\gamma \delta n)$ steps.
\end{proof}


\section{Conclusions}
The task of re-optimizing a previously encountered problem plays a crucial role in real-world applications. We contribute to the theoretical understanding and design of evolutionary algorithms for such dynamically changing problems and introduced a diversity-based approach which searches for good solutions around a good solution prior to the perturbation. This allows the algorithm to remember good components of the given problem. Our theoretical results show that this leads to highly effective evolutionary algorithms as it prevents recomputation of previously obtained knowledge about the given problem.

\subsection*{Acknowledgments}
This work has been supported by the Australian Research Council through grants DP160102401 and DP190103894, by COST action CA15140 (`ImAppNIO'), and by a public grant as part of the Investissement d'avenir project, reference ANR-11-LABX-0056-LMH, LabEx LMH, in a joint call with Gaspard Monge Program for optimization, operations research and their interactions with data sciences. Parts of this research has been conducted during a research visit of Frank Neumann as invited professor at Sorbonne University, with financial support from the LIP6 laboratory. 

\newcommand{\etalchar}[1]{$^{#1}$}

\end{document}